\documentclass{article}

\usepackage{xcolor}
\usepackage[T1]{fontenc}

\usepackage{microtype}
\usepackage{graphicx}
\usepackage{subfigure}
\usepackage{booktabs} 
\usepackage{hyperref}

\makeatletter
\@namedef{ver@natbib.sty}{9999/12/31}
\let\setcitestyle\@gobble
\usepackage[accepted]{icml2020}
\let\setcitestyle\undefined
\expandafter\let\csname ver@natbib.sty\endcsname\@undefined
\makeatother

\usepackage{csquotes}
\usepackage[backend=biber]{biblatex}
\newcommand{\citet}[1]{\textcite{#1}}
\newcommand{\citep}[1]{\parencite{#1}}

\icmltitlerunning{Neural Variational Gradient Descent}

\usepackage{amsmath,amssymb, mathrsfs, amsthm}

\usepackage{hyperref}
\usepackage{pgf}
\usepackage{enumerate}
\usepackage{comment}

\DeclareMathOperator*{\argmax}{arg max}
\DeclareMathOperator*{\argmin}{arg min}

\DeclareMathOperator{\dif}{d \!}
\DeclareMathOperator{\dv}{div}

\DeclareMathOperator{\KL}{KL}
\DeclareMathOperator{\SD}{SD}
\DeclareMathOperator{\RSD}{RSD}

\DeclareMathOperator{\med}{med}
\DeclareMathOperator*{\ev}{E}

\newcommand{\e}{\varepsilon}
\newcommand{\isum}{\sum_{i=1}^n}

\newcommand{\R}{\mathbb{R}}

\newcommand{\X}{\mathcal X}

\newcommand\note[1]{\textcolor{red}{#1}}

\newcommand{\hs}[1]{\textcolor{red}{\textbf{hs}: #1}}
\newcommand{\vf}[1]{\textcolor{blue}{\textbf{vf}: #1}}
\newcommand{\la}[1]{\textcolor{cyan}{\textbf{ll}: #1}}

\renewcommand\note[1]{}

\renewcommand{\hs}[1]{}
\renewcommand{\vf}[1]{}
\renewcommand{\la}[1]{}

\theoremstyle{plain}
\newtheorem{theorem}{Theorem}[section]

\newtheorem{proposition}[theorem]{Proposition}
\theoremstyle{remark}

\theoremstyle{definition}

\begin{document}

\twocolumn[
\icmltitle{Neural Variational Gradient Descent}

\icmlsetsymbol{equal}{*}

\begin{icmlauthorlist}
\icmlauthor{Lauro Langosco di Langosco}{eth}
\icmlauthor{Vincent Fortuin}{eth}
\icmlauthor{Heiko Strathmann}{dm}
\end{icmlauthorlist}

\icmlaffiliation{eth}{ETH Zurich, Zurich, Switzerland}
\icmlaffiliation{dm}{Deepmind, London, United Kingdom}

\icmlcorrespondingauthor{Lauro Langosco di Langosco}{langosco.lauro@gmail.com}

\vskip 0.3in
]

\printAffiliationsAndNotice{}

\begin{abstract}

Particle-based approximate Bayesian inference approaches such as Stein Variational Gradient Descent (SVGD) combine the flexibility and convergence guarantees of sampling methods with the computational benefits of variational inference.
In practice, SVGD relies on the choice of an appropriate kernel function, which impacts its ability to model the target distribution---a challenging problem with only heuristic solutions.
We propose Neural Variational Gradient Descent (NVGD), which is based on parametrizing the witness function of the Stein discrepancy by a deep neural network whose parameters are learned in parallel to the inference, mitigating the necessity to make any kernel choices whatsoever.
We empirically evaluate our method on popular synthetic inference problems, real-world Bayesian linear regression, and Bayesian neural network inference. 

\end{abstract}

\section{Introduction}
This work is concerned with the problem of generating samples from an unnormalized Bayesian posterior. In particular, we are interested in the case where the target posterior is estimated from a large dataset, and only stochastic (minibatch) approximations are available.

In this setting, standard MCMC algorithms such as Hamiltonian Monte Carlo \citep{neal_hmc} or MALA \citep{roberts1998} face difficulties \citep{betancourt_incompatibility, roberts1998}. As a consequence, practitioners tend to prefer simpler methods such as variational inference (VI) and stochastic gradient Langevin dynamics (SGLD) \citep{sgld}. A more recent method developed by \citet{svgd}, Stein variational gradient descent (SVGD), has been gaining popularity.

\begin{figure}[ht]
\centering
\includegraphics[width=0.8\columnwidth]{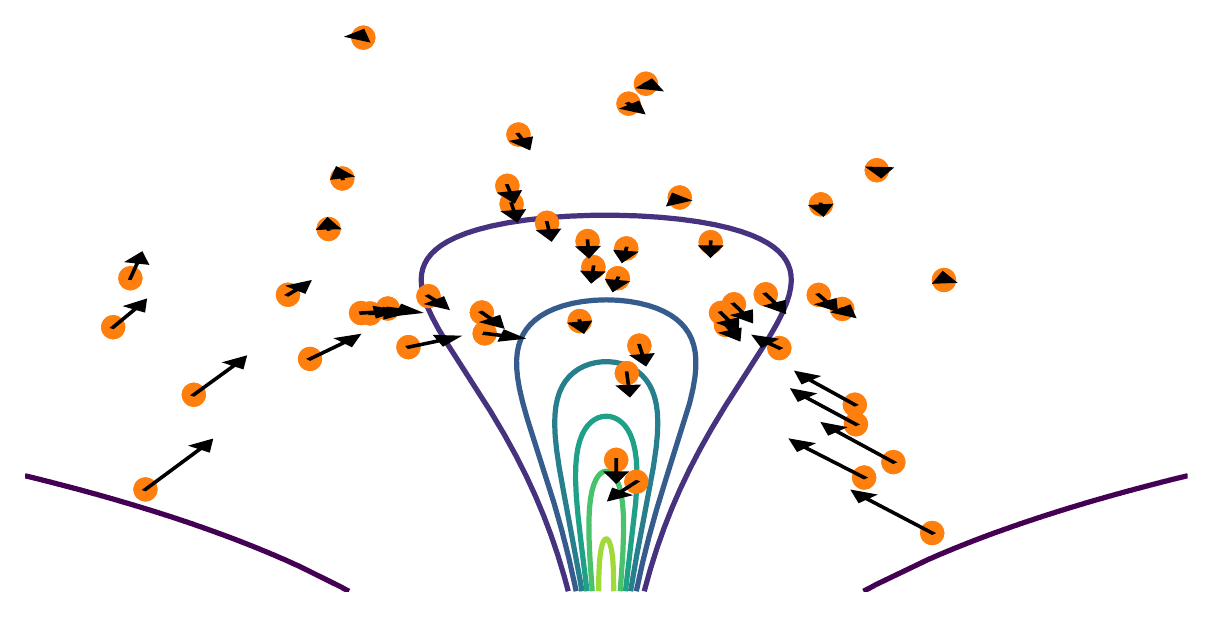}
\caption{
\textbf{Neural Variational Gradient Descent:} Initialize a set of samples $x_1, \dots, x_n$ (orange). 
At each iteration, update samples via $x_i = x_i + \e f_\theta(x_i)$ (black arrows), where the neural network $f_\theta$ is trained to minimize the KL divergence from the target to the particles.
In contrast to MCMC methods, NVGD is deterministic and can take advantage of `between-chain' interactions.
}
\end{figure}

These methods have in common that they minimize the KL divergence $\KL(q \ \Vert \ p)$ between an approximating distribution $q$ and the target $p$: 
VI does so explicitly by optimizing a parameterized density $q_\theta$, thus necessitating the choice of a (more or less) limited variational family of densities, which can introduce bias.
SGLD and SVGD do so non-parametrically by producing samples whose distribution follows a gradient-descent-like trajectory towards the target \citep{jko, liu2017stein}. 

In particular, SGLD approximates a gradient flow in the Wasserstein space of probability measures.
This gradient flow can also be approximated deterministically by evolving a set of particles using SVGD \citep{svgd}, which allows to take the geometry of the target space into account.
However, this geometry has to be encoded in a kernel function within the Stein class, which is often challenging to choose in practice.

In this work, we aim to overcome those limitations by learning an update function $f_\theta$ parameterized by a neural network that is trained in parallel to the inference. The samples are then updated via $x \leftarrow x + \varepsilon f_\theta(x)$.
This approach has the following benefits:
\begin{itemize}
\item Unlike VI, our method does not constrain inference to a fixed family of distributions.
\item Unlike SGLD, our method is deterministic and adapts to the geometry of the target posterior.
\item Unlike SVGD, it is not dependent on a choice of kernel or kernel parameters and thus can \emph{automatically} adapt to the geometry of the target.
\end{itemize}
Empirically, we observe that our method outperforms or matches SVGD for a large array of choices of kernel parameters, while itself not relying on a kernel. In addition, we observe in synthetic experiments that our method exhibits less asymptotic bias than Langevin dynamics or SVGD.

\section{Inference via KL-Minimizing Flows} 
\label{sec:inference-via}

\citet{jko} famously showed that many variational schemes can be viewed as a gradient flow in the space of probability measures equipped with the Wasserstein metric. In this section, we will give a very short overview of the results relevant for our setting.

Let $p \in \mathcal P(\R^n)$ be the target posterior we wish to sample from.
We start by casting the problem of sampling as an optimization problem: the goal is to generate samples distributed according to the minimizer $q^*$ of the KL divergence,
\begin{equation*}
  q^* = \argmin_{q \in \mathcal P'} \KL(q \ \Vert \ p),
\end{equation*}
where $\mathcal P'$ is some appropriate space of candidate measures. 
In the following we will consider the space $\mathcal P$ of probability measures on $\R^n$ equipped with the $2$-Wasserstein metric. This metric allows the definition of a \emph{Wasserstein gradient flow}: a continuous trajectory $(q_t)_{t \geq 0}$ in $\mathcal P$ that locally minimizes a functional $F(q_t)$.
For details on Wasserstein gradient flows we refer to textbooks \parencite[e.g.][]{ot,ambrosio_flows}.
We are interested in the case where this functional is the KL divergence: $F(q) = \KL(q \ \Vert \ p)$. In terms of samples $X_t \sim q_t$, this KL-minimizing flow is characterized by the deterministic Markov process $(X_t)_{t \geq 0}$ given by
\begin{equation} \label{eq:klgrad}
  \frac{dX_t}{dt} = \nabla \log p(X_t) - \nabla \log q_t(X_t),
\end{equation}
where $X_0 \sim q_0$ is initialized according to some initial distribution $q_0$, and $q_t$ denotes the density of $X_t$.
Under reasonable conditions, this process converges in the sense that $\lim_{t \to \infty} \KL(q_t \ \Vert \ p) = 0$ \citep{convergence_ula}.
The direct application of the Euler discretization
\begin{equation*}
  x_{k+1} = x_k + \e \Big(\nabla \log p(x_k) - \nabla \log q_k(x_k) \Big)
\end{equation*}
as a sampling method is intractable because it is usually not possible to efficiently compute $q_k(x)$.\footnote{For an invertible, differentiable transformation $T$, the pushforward distribution of $q$ under $T$ can be computed via the change of variables formula $T_\#q(x) = q(T^{-1} x) \cdot \det \vert J_{T^{-1}}(x) \vert$. The computation of the determinant of a general $d \times d$ matrix needs $\mathcal O(d^3)$ operations; when $d$ is large, this is too expensive.}
Instead, the usual discretization is Langevin dynamics (SGLD), which \citet{jko} have shown to be equivalent (in the limit of infinitesimal step-size) to the flow~\eqref{eq:klgrad} in the sense that the marginal distributions $(q_t)_{t \geq 0}$ are equal.

We take a different approach. By evolving many samples in parallel, we estimate the gradient using a neural network $f_\theta$ and then perform the update
\begin{equation}
  x_{k+1} = x_k + \e f_\theta(x_k)
\end{equation}

\section{Neural Variational Gradient Descent} \label{sec:method}
In this section, we will see how to train a neural network $f_\theta$ to approximate the gradient $\nabla \log p(x_k) - \nabla \log q_k(x_k).$
We build on the method of Stein variational inference developed by \citet{svgd}.
In a slightly non-standard manner, let us define the Stein discrepancy \citep{sd} given an $L^2(q)$-measureable vector field $f: \R^n \to \R^n$ as
\begin{align}
  \SD(q \ \Vert \ p; f) =& \ev_{x \sim q} \Big[ f(x)^T \big(\nabla \log p(x) - \nabla \log q(x) \big) \Big] \nonumber \\
  =& \ev_{x \sim q} \Big[ f(x)^T \nabla \log p(x) + \dv f(x) \Big], \label{eq:sd}
\end{align}
where the second equality follows via partial integration (details in Appendix~\ref{app:sd}).
Here, $\dv f = \sum_{i=1}^d \frac{\partial f_i}{\partial x_i}$ denotes the divergence of $f$ (the trace of $f$'s Jacobian).

There are two properties of the Stein discrepancy that deserve special emphasis:
Firstly, it can be estimated using only samples from $q$, without access to the explicit form of the density.
Secondly, a regularized Stein discrepancy is maximized by the Wasserstein gradient from Equation~\eqref{eq:klgrad}.
Indeed, following \citet{lsd}, let us define the \emph{regularized Stein discrepancy} as
\begin{equation} \label{eq:rsd}
  \RSD(f_\theta) = \SD(q \ \Vert \ p; f_\theta) - \frac{1}{2} \Vert f_\theta \Vert^2_{L^2(q)},
\end{equation}
where $\Vert f_\theta \Vert^2_{L^2(q)} = \ev_{x \sim q}[f_\theta(x)^Tf_\theta(x)]$. Now it is easy to confirm that
\begin{equation*}
  f^* = \argmax_{f \in L^2(q)} \RSD(q \ \Vert \ p; f) = \nabla \log p - \nabla \log q.
\end{equation*}
This choice of regularization is equivalent to bounding the total squared distance by which the particles are moved.
The regularization term is necessary because otherwise the optimization could scale $f$ by an arbitrarily large factor.

\subsection{Computation}
Let $x_1, \dots, x_n$ be a set of samples which we wish to transport towards the target posterior $p$. In this section, we describe how to compute a single iteration of our method; the entire procedure is summarized in Algorithm~\ref{alg:main}.

While it is intractable to compute $\RSD(f)$ exactly, we can approximate it with the Monte Carlo estimate
\begin{multline*}
  \widehat{\RSD}(f) = \frac{1}{n} \isum f(x_i)^T \nabla \log
  p(x_i) + \dv f(x_i) \\ - \frac{1}{2} f(x_i)^T f(x_i).
\end{multline*}
If the dimension $d$ of the sample space is large, the $O(d^2)$ computation of the divergence $\dv f$ might be too expensive. We follow \citet{lsd} in using Hutchinson's estimator \citep{hutchinson}, which gives an efficient and unbiased estimate of $\dv f$ via
\[
  \dv f(x) \approx z^T \nabla f(x) z, \ \ z \sim \mathcal{N}(0, 1),
\]
which is derived from the identity
\[
  \dv f(x) = \ev_{z \sim \mathcal{N}(0, 1)}[z^T \nabla f(x) z],
\]
and has been used in a number of recent works \citep{lsd, han2017, grathwohl_ffjord_2018}.

One iteration of our method then takes two steps: 
\begin{enumerate}
\item Train the model $f_\theta$ to maximize $\widehat{\RSD}$ for a fixed number of steps, and
\item update the particles using the trained model. 
\end{enumerate}
The full algorithm is described in Algorithm~\ref{alg:main}. To prevent the model from overfitting, we track the value of the $\widehat{\RSD}$ on a validation set of particles and early-stop the training if the validation $\widehat{\RSD}$ stops increasing.

\begin{algorithm}[tb]
   \caption{Neural variational gradient descent\hs{unclear what RSD means when just reading the algo}}
   \label{alg:main}
\begin{algorithmic}
\STATE {\bfseries Input:} Learning rates $\eta$ and $\e$, number of gradient updates $P$
\STATE Initialize parameters $\theta$ and particles $x_1, \dots, x_n$
\REPEAT
\STATE $X_\text{train}, \ X_\text{validation} \gets \text{RandomSplit}(x_1, \dots, x_n)$
\FOR{$P$ steps}
\STATE $\theta \gets \theta + \eta \nabla_\theta \widehat{\RSD}(\theta; X_\text{train})$
\IF{EarlyStopCondition$(X_\text{validation})$}
\STATE {\bfseries break}
\ENDIF
\ENDFOR
\STATE $x_i \gets x_i + \e f_\theta(x_i)$ for all $1 \leq i \leq n$
\UNTIL{Converged}
\end{algorithmic}
\end{algorithm}

\subsection{Computational Complexity}
In large-data applications, the bottleneck of our method is the computation of the posterior $\nabla \log p$, which needs to be computed $n$ times per step. 

The same bottleneck applies to competing methods such as SGLD, SVGD, and variational inference, which makes the methods similar as far as computational cost \emph{per step} is concerned. 
For a computational comparison between the methods, the main factor is therefore the number of iterations needed to converge.
Experimentally, we show that our method converges in a similar or smaller number of steps than both SGLD and SVGD.

\section{Related work} \label{sec:related}
\paragraph{SVGD.}
Our method is based on the work of \citet{svgd}, who developed the SVGD algorithm.
It has been previously noted that this algorithm often fails when the target distribibution is high-dimensional and has a complex shape. Attempted remedies include lower-dimensional projections \citep{message_passing_svgd, sliced_ksd, chen2020projected} or modifications to the update equations \citep{gallego2018stochastic, chewi2020,dangelo2021annealed}.
The main difference between our method and these ones is that they only consider perturbations within an RKHS, thus requiring the explicit choice of a kernel function and restricting the flexibility of the method, while our method does not require any kernel choice whatsoever.

\paragraph{Learning the witness function.}
The approach used in this paper for learning the Stein discrepancy is modeled after the approach used by \citet{lsd} to train and evaluate energy-based models. Another work that uses a similar approach to learn a Stein discrepancy is \citet{stein-neural-sampler}. Both of these works directly minimize the Stein discrepancy between an approximating distribution (an energy-based model or a set of samples, respectively) and a target. In contrast, our method minimizes the KL-divergence by leveraging Equation~\ref{eq:sd}, using an implicit gradient obtained via the Stein discrepancy.

\paragraph{Extensions.}
Recently, many extensions have been proposed for the SVGD algorithm, including using projections \citep{chen2020projected}, stochastic disrepancies \citep{gorham2020stochastic}, second-order methods \citep{detommaso2018stein}, additional Langevin noise \citep{gallego2018stochastic}, importance weighting \citep{yao2020stacking}, non-Markovian updates \citep{ye2020stein}, explicit noise models \citep{chang2020kernel}, and temperature annealing schedules \citep{dangelo2021annealed}.
While we focus on the simplest setting as a proof of concept in this work, most of these ideas do not rely on the kernelized form of the update and could also be readily applied to our method.
We leave this as an avenue for future work.

\section{Experiments} \label{sec:experiments}
In this section, we test our NVGD on various synthetic benchmarks, one large-data logistic regression task, and Bayesian neural network (BNN) inference.
\footnote{Code for all experiments is available at \url{https://github.com/langosco/neural-variational-gradient-descent}}
For most experiments we choose $f_\theta$ to be a simple neural network with two linear hidden layers of size $32$ and an output layer of size $d$, where $d$ is the dimensionality of the sample space. An exception is the BNN task, where we choose a larger network with three hidden layers and $256$ units per layer. \hs{reviewer question: how does the network architecture impact performance? Also this network is fairly simple, does this mean the problems are also very simple? Is it even feasible to learn larger numbers of network parameters?} \la{a large difference in the BNN experiment. for smaller experiments I expect the difference to be small; havent tested this yet though.}


We will compare NVGD with two closely related sampling algorithms: SVGD \citep{svgd} and the unadjusted Langevin algorithm (ULA). Both are approximations of the Wasserstein gradient flow of the KL divergence: SVGD is a particle-based algorithm that at every step applies the particle transformation in a reproducing kernel Hilbert space that locally minimizes the KL divergence. ULA is an MCMC algorithm that has been shown to also minimize the KL divergence. In the logistic regression and the BNN experiments we estimate the posterior density using minibatches; in this setting, ULA is referred to as stochastic gradient Langevin dynamics (SGLD) \cite{sgld}.

For SVGD, one has to choose a kernel. The most frequent choice in the literature is the squared-exponential kernel $k(x, y) = \exp(- \Vert x - y \Vert^2 / (2 h^2))$ with bandwidth $h$ chosen according to the median heuristic\hs{would be nice to say this in the intro and say that this kernel is really unsuited for many problems due to: stationarity, infinite smoothness, etc.}
\[
  h^2 = \frac{\med^2}{2 \log n},
\]
where $\med^2$ is the median of all squared pairwise distances $\Vert x_i - x_j \Vert^2$ between samples $x_1, \dots, x_n$. We follow this choice in all experiments.

\begin{figure}[t]
\includegraphics[width=\columnwidth]{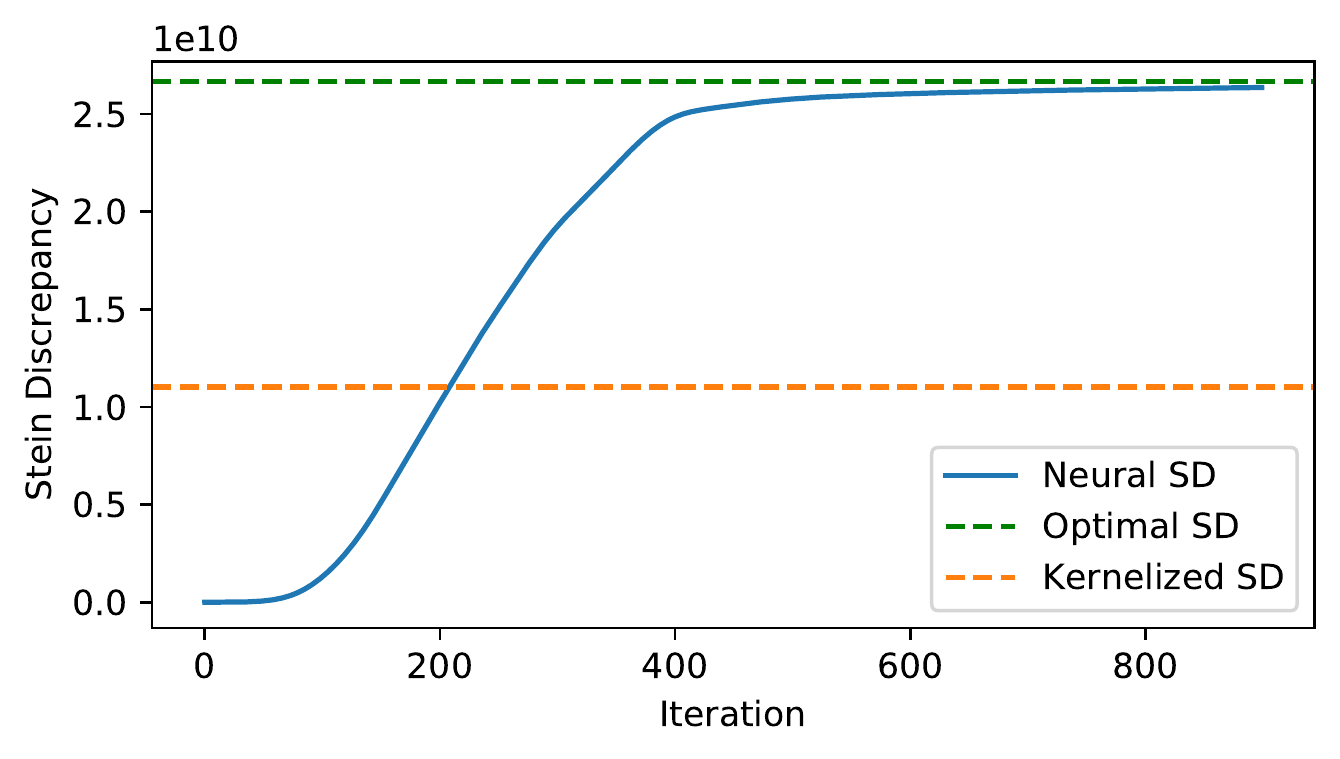}
\caption{The Stein discrepancy over the course of gradient descent on $\theta$. Larger is better, since the Stein discrepancy is proportional to the amount that one particle iteration reduces the KL divergence (Equation~\ref{eq:sd}). Our proposed NVGD converges closely to the optimal Stein discrepancy.}
\label{fig:sd-maxing}
\end{figure}

\subsection{Minimizing the KL in one step}
Our first experiment is a sanity check: we confirm that, on a synthetic task, our method approximates the true gradient of the KL, that is $f_\theta(x) \approx f^*(x) = \nabla \log p(x) - \nabla \log q(x)$.

In the synthetic examples below we have access to both $p$ and $q$ (which is not the case in real applications) and can thus compute $f^*$ directly. We confirm that our method successfully learns to approximate $f^*$.

Results are shown in Figure~\ref{fig:sd-maxing}.
For comparison, we also plot the kernelized Stein discrepancy, rescaled such that the kernelized (SVGD) update has $L_2$-norm equal to $f^*$ (rescaling is necessary because the scale of the update, which corresponds to the particle step size, is a priori arbitrary).

We choose the target posterior $p$ to be an ill-conditioned Gaussian in $\R^{50}$ with mean zero and a diagonal covariance matrix with entries spaced logarithmically from $10^{-4}$ to $1$.

We sample a batch of $1000$ particles from a `proposal' distribution $q$, here chosen to be a standard Gaussian, and train the network $f_\theta$ to maximize the rescaled Stein discrepancy~\eqref{eq:rsd} for $1000$ iterations. As visible in Figure~\ref{fig:sd-maxing}, the network quickly learns to match the Wasserstein gradient.

\begin{figure}[t]
\includegraphics[width=\columnwidth]{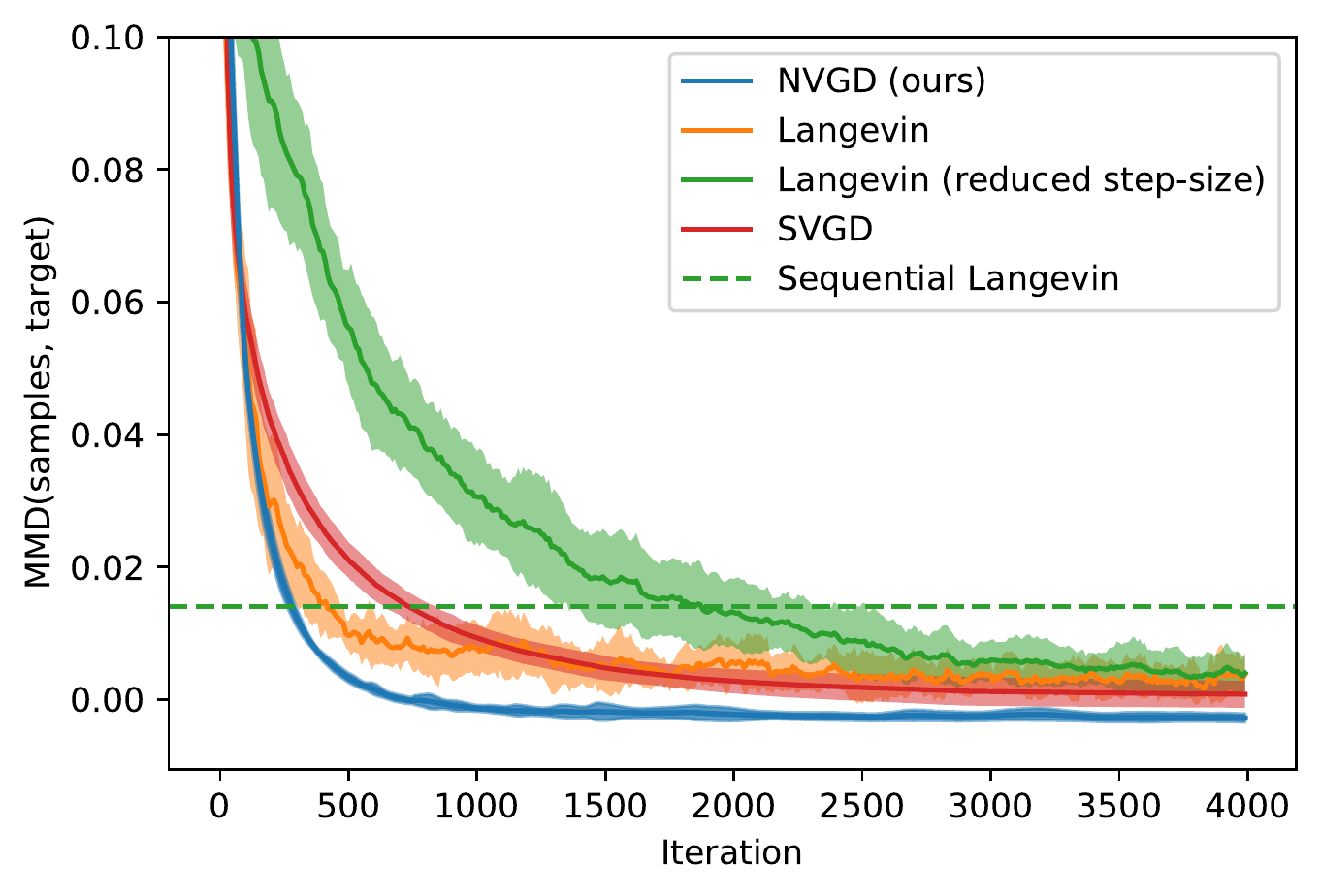}
\caption{We compare NVGD, SVGD, and Langevin dynamics on the funnel density. If we use the same step-size for NVGD and pULA, then the convergence is fast, but the asymptotic error is larger for pULA. If the pULA step-size is chosen smaller, then it converges as well as NVGD, but more slowly. 
We plot the mean and standard deviation across ten runs.
The step-size of SVGD is tuned to be maximal while still converging to low MMD error. 
}
\label{fig:funnel-mmd}
\end{figure}

\subsection{Synthetic benchmark distributions}
Neal's Funnel \citep{neal_funnel} is a distribution on $\R^d$ defined via the density
\begin{equation}
  \label{eq:funnel}
  p(x) = N(x_1; 0, 3) \cdot \prod_{i=2}^d N(x_i; 0, \exp(x_1)).
\end{equation}
This density is a common benchmark for sampling algorithms because it is easy to sample from directly (and thus easy to track convergence) but also challenging due to its geometry.

We choose $d=2$ and initialize a set of $100$ particles by sampling from a standard Gaussian and develop them along the approximate gradient flows given by the different methods. We track convergence using Maximum Mean Discrepancy (MMD) \citep{mmd} with a squared-exponential kernel. Figure~\ref{fig:funnel-mmd} shows how the methods converge, measured by MMD, and Figure~\ref{fig:funnel-scatter} is a scatterplot that shows why the funnel is a challenging target.

\begin{figure}[t]
\resizebox{\linewidth}{!}{
\input{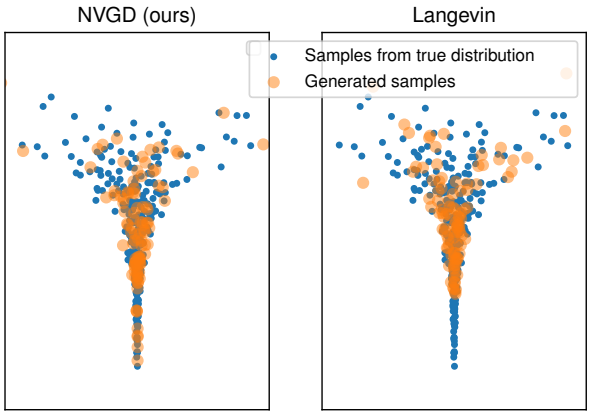}
}
\caption{Samples generated for the Funnel density (same setting as Figure~\ref{fig:funnel-mmd}). The difficulty arises from the funnel geometry: too wide at the top, and too narrow at the bottom. In particular, Langevin dynamics (pULA) is not able to sample from the narrow strait if the step size is too large.}
\label{fig:funnel-scatter}
\end{figure}

ULA and NVGD are discretizations (stochastic and deterministic, respectively) of the Wasserstein gradient flow~\eqref{eq:klgrad}. It is known that in the limit of infinitesimally small step-size, ULA perfectly approximates this flow. Any hope of outperforming ULA must therefore lie in a better approximation of the flow \emph{relative to the step-size}, that is, we need to show that ULA with a step-size $\eta > 0$ is either more biased than NVGD (if $\eta$ is too large) or takes more steps to converge than NVGD (if $\eta$ is too small).

When comparing the methods on this task, where it is possible to track convergence exactly, we find that NVGD outperforms ULA. In particular, we compare with two versions of ULA: the standard single-chain ULA, and a parallel version pULA where $n$ chains are initialized and evolved simultaneously. The parallel version is more similar to NVGD, which also transports a number of particles in parallel. 

To reduce correlation between sequential ULA samples, most are discarded (e.g., only every 100th sample is retained). We develop a single chain for $5000 \cdot 100 = 5 \cdot 10^5$ steps, corresponding to the same number of likelihood evaluations as $5000$ steps for $100$ parallel chains. 
However, even after discarding all but every $100$th sample, the resulting samples are still correlated.
This results in a high MMD error (Figure~\ref{fig:funnel-mmd}). For sequential ULA, we use the same learning rate as the pULA run that achieved lower MMD error.
When using a higher learning rate instead, we found that the samples were less correlated at the cost of worse coverage of the posterior.

In many applications, ULA is augmented with an accept/reject step that adjusts for its asymptotic bias. We compare to the \emph{unadjusted} version because the accept/reject step is not used in large-data applications (recent work by \citet{garriga2021exact} has shown it possible to do so using custom integrators, but this is beyond the scope of this work). Instead, the stochastic gradient variant of ULA (SGLD) is preferred, as for instance in \citet{wenzel_cold_posteriors}. Where an accept/reject step is feasible, NVGD can be similarly augmented, though we do not study this here.

\begin{figure}[t]
\includegraphics[width=\columnwidth]{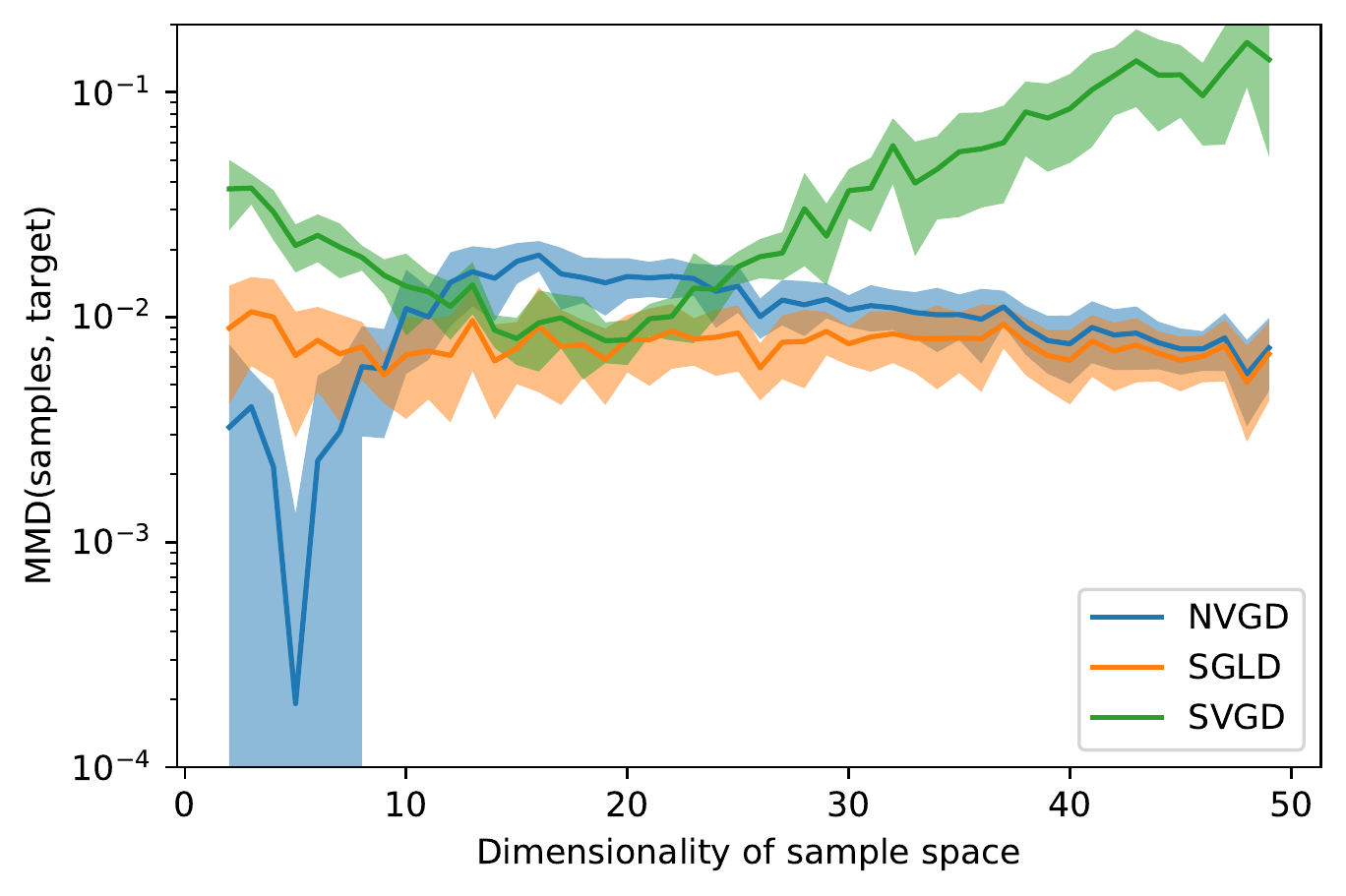}
\caption{Results of sampling from Neal's Funnel in dimensions $2$--$40$. Lower is better.
Despite what this figure appears to show, the performance of NVGD and ULA also degrades with increasing dimensionality of the target.
We plot the mean and standard deviation across ten runs.
}
\label{fig:funnel-sweep}
\end{figure}

\paragraph{Higher-dimensional funnels.}
We test performance of NVGD for the Funnel distribution across $40$ dimensions. Figure~\ref{fig:funnel-sweep} shows how the performance of SVGD degrades as the dimensionality increases. In other settings, it has already been shown that high-dimensional sample spaces are often challenging for SVGD \citep{message_passing_svgd, sliced_ksd}, though in general this depends on the choice of kernel function.
Figure~\ref{fig:funnel-sweep} also seems to show that the SGLD and NVGD are immune to the difficulties of increasing dimensionality. Unfortunately, visual inspection of samples shows that this is just an artefact of the MMD measure, and the performance of the methods does get worse as the dimensionality increases.

One thing that is surprising about Figure~\ref{fig:funnel-sweep} is that ULA performs equally or better than NVGD, which seems to contradict the results in Figure~\ref{fig:funnel-mmd}. This discrepancy likely stems from two sources: 1) ideally, the learning rate of $f_\theta$ and the size of its hidden layers should be adapted as the dimension increases, but we keep it fixed; and 2) the advantage of NVGD---better coverage along the narrow end of the $y$-axis---becomes less relevant for the MMD measure as the dimension increases, since all $d-1$ other axes are symmetric.

\begin{figure}[t]
\includegraphics[width=\columnwidth]{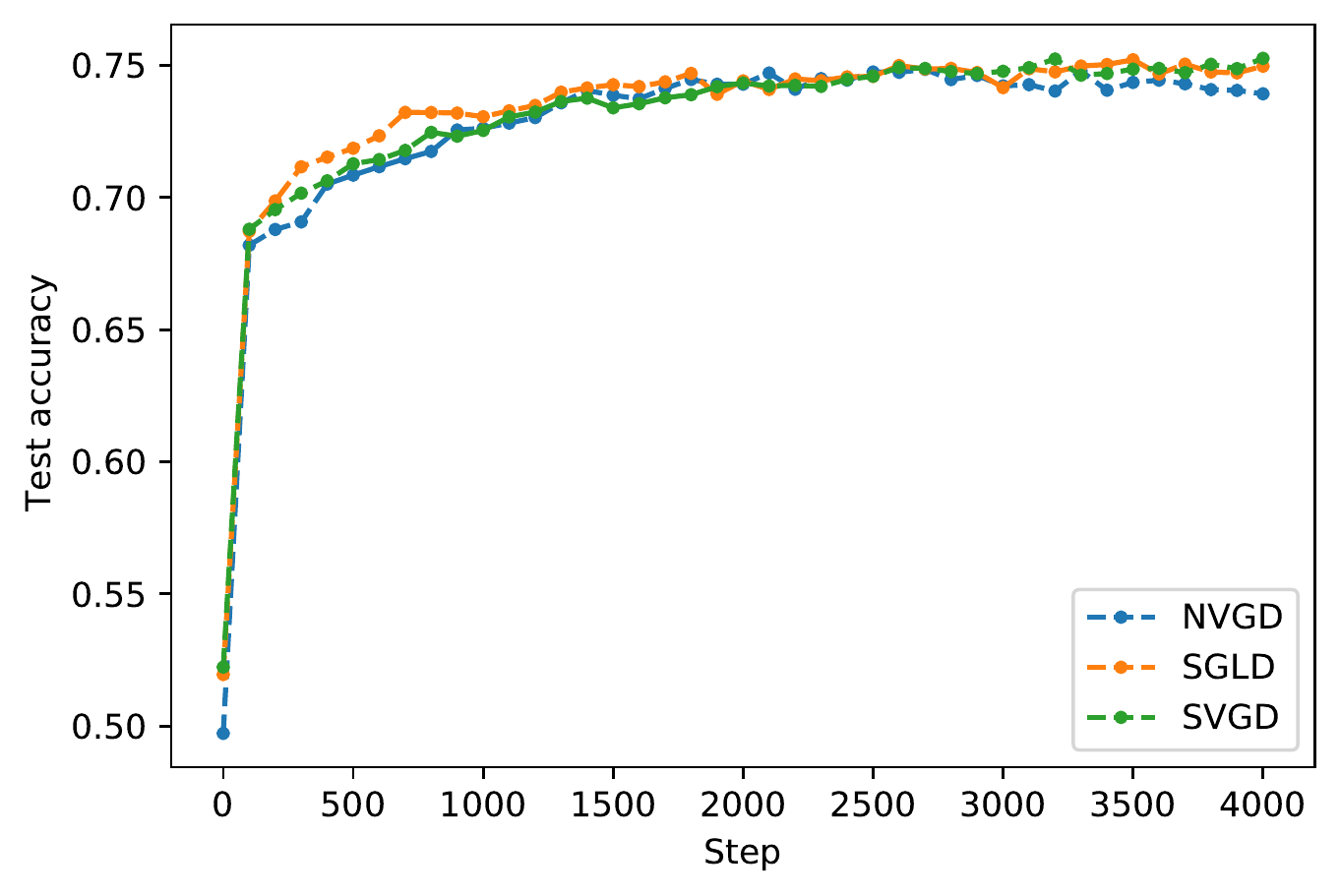}
\caption{Bayesian logistic regression on the Covertype dataset. We use $100$ particles and a mini-batch size of $128$, so $4000$ steps corresponds to one epoch. Results are averaged over ten runs.}
\label{fig:covertype}
\end{figure}

\begin{figure}[t]
\includegraphics[width=\columnwidth]{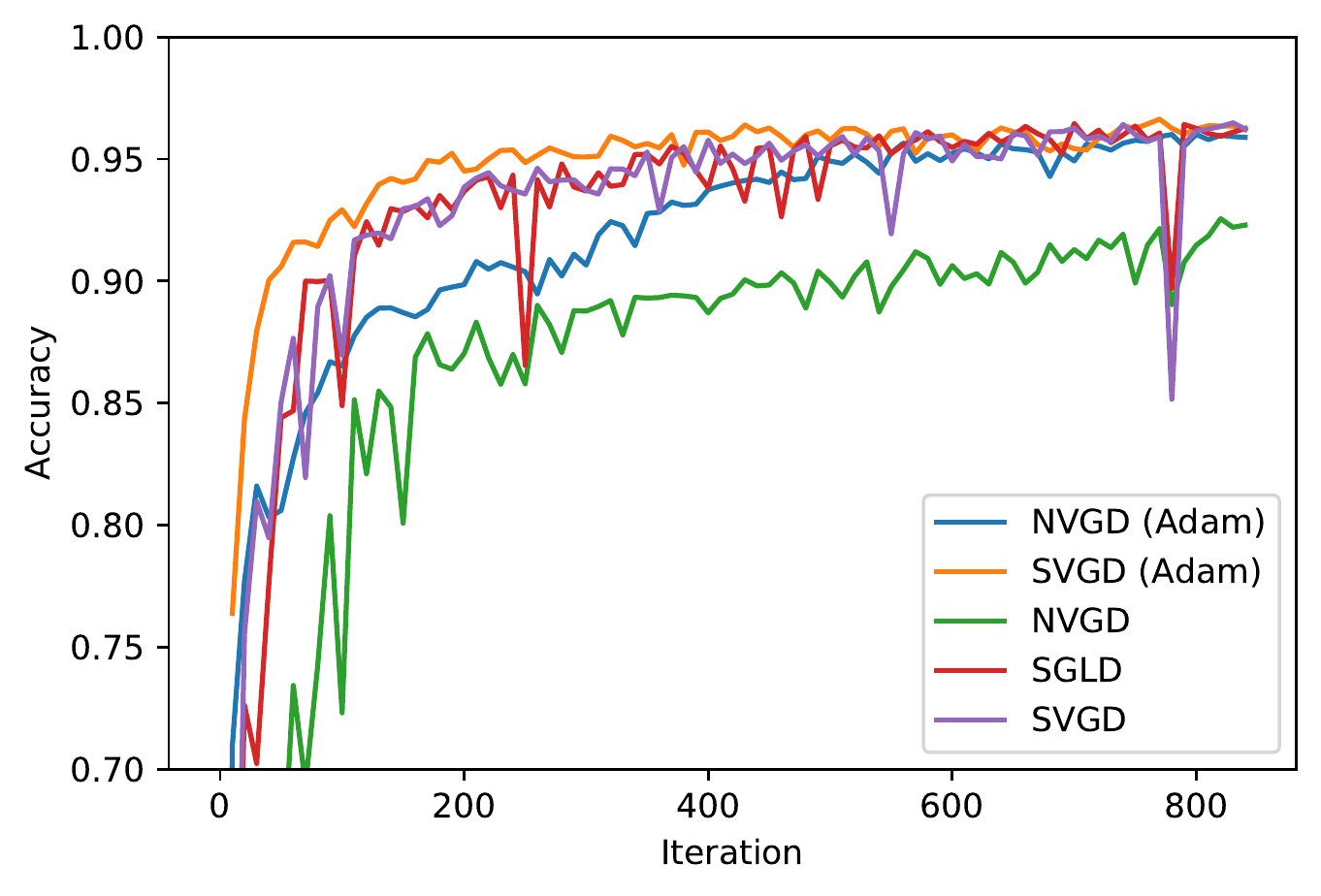}
\caption{Bayesian neural network inference on MNIST. The vanilla version of our method (NVGD) performs poorly; when combined with the Adam optimizer, it performs better, but still converges more slowly than either SVGD and SGLD. One epoch corresponds to $420$ iterations.}
\label{fig:bnn}
\end{figure}

\subsection{Bayesian logistic regression}
We sample from a Bayesian logistic regression model trained on the binary Covertype dataset\footnote{available at \url{https://www.csie.ntu.edu.tw/~cjlin/libsvmtools/datasets/binary.html}} \citep{covertype}.
This dataset consists of $581,012$ observations of $55$ features and a binary label. We split off a proportion of $0.2$ for use as test set.
We follow~\citet{svgd} in using a hierarchical prior on $\beta$:
\begin{align*}
  \alpha \sim \ & \text{Gamma}(a_0, b_0), \\
  \beta \sim \ & \mathcal{N}(0, \alpha^{-1}),
\end{align*}
choosing the rate and inverse scale parameters as $a_0 = 1$ and $b_0 = 0.01$. The logistic model for $(x, y)$ is then
\begin{equation*}
  y \sim  \text{Bernoulli}\left(\frac{\exp(x^\top\beta)}{1 + \exp(x^\top\beta)}\right).
\end{equation*}
To test our method in the stochastic gradient setting, we estimate the likelihood using minibatches of size $128$ and train for one epoch ($4085$ steps).

We compare again the three methods NVGD, SGLD, and SVGD. For all methods, we sample $100$ particles and choose the particle step-size via tuning on a validation set of size $0.1$ subsampled from the training set. We see in Figure~\ref{fig:covertype} that all three methods perform similarly. The test accuracy is averaged over ten runs.
\vf{can we discuss why this is? is the posterior too simple? can we maybe visualize it somehow?}

\subsection{Bayesian Neural Network}

While Bayesian neural networks (BNNs) \citep{mackay1992practical, neal1992bayesian} have gained a lot of popularity recently \citep{immer2021improving, immer2021scalable,  fortuin2021bayesian, fortuin2021bnnpriors, fortuin2021priors, izmailov2021bayesian, wenzel_cold_posteriors}, SVGD has been used for their inference on only few occasions \citep{wang2018function, hu2019applying, dangelo2021stein, dangelo2021repulsive}, possibly due to the difficulty to adapt to their high-dimensional posterior geometry.
We train a BNN on the MNIST dataset and again compare against SVGD and SGLD. The Bayesian classifier network consists of two convolutional and one fully-connected layer with $4594$ parameters in total. Results are shown in Figure~\ref{fig:bnn}.

For MNIST we use a batch size of 128. We tune the BNN learning rate to maximize accuracy on a validation set comprising $10\%$ of the training set, and apply the Adam optimizer to the SVGD and NVGD updates (this cannot easily be done for the SGLD gradients, so we use vanilla SGLD; for comparison, we also include results for vanilla SVGD and NVGD).

While all three methods reach similar accuracy, NVGD (our method) converges more slowly. Closer observation suggests that the gradient learner $f_\theta$ underfits; future work will consider architectures other than the simple linear feedforward network used here.

\section{Conclusion} \label{sec:conclusion}

We introduced NVGD, a novel particle-based variational inference method that transports particles from a reference distribution to the target posterior along a dynamically learned trajectory.
In comparison to the two most relevant competitor methods (SVGD and SGLD), our method outperforms them on synthetic benchmarks, but slightly underperforms on Bayesian neural network inference.
We hypothesize that this is because the current architecture underfits in complex sample spaces (such as the parameter space of a Bayesian neural network), and that the method can be improved by using larger or more complex architectures.


\printbibliography

\newpage
\onecolumn

\appendix

\section{Proofs}
\subsection{The Stein Discrepancy} \label{app:sd}
In Equation~\ref{eq:sd} we defined the Stein discrepancy given $f$ as
\begin{align}
  \SD(q \ \Vert \ p; f) =& E_{x \sim q}[f(x)^T(\nabla \log p(x) - \nabla \log q(x))] \nonumber \\
  =& E_{x \sim q}[f(x)^T \nabla \log p(x) + \dv f(x)].
\end{align}
We will now provide the assumptions under which the second equality holds, and show how it follows from integration by parts.
\begin{proposition} \label{prop:div}
  Let $q$ be a probability density on $\X$ and $f: \X \to \R^d$, $f \in L^2(q)$ a differentiable function. Assume that 
\begin{equation}
\int_{\partial \X} f(x) q(x) \dif x = 0,
\end{equation}
where $\partial \X$ denotes the boundary of $\X$. If $\mathcal X$ is instead all of $\mathbb R^d$, then the condition must hold in the limit $r \to \infty$ for integral over the ball $B_r$ of radius $r$ centered at the origin. Then
\begin{equation*}
  E_{x \sim q}[\dv f(x)] = - E_{x \sim q}[f(x)^T \nabla \log q(x)].
\end{equation*}
\end{proposition}
\begin{proof}
  The proposition follows directly from integration by parts. If $n(x)$ is the outward-pointing unit vector on the boundary of $\X$, then
  \begin{align*}
    E[f(x)^T \nabla \log q(x)] &= \int_\X f(x)^T \nabla q(x) \dif x \\
                               &= \int_{\partial \X} f(x)^T n(x) q(x) \dif x - \int_\X \dv f(x) q(x) \dif x \\
                               & = - E[\dv f(x)].
  \end{align*}
Our desired equality follows.
\end{proof}
Proposition~\ref{prop:div} is used in many problems throughout statistics and deep learning. It is the basis of the policy gradient algorithm in reinforcement learning and the score matching algorithm for density estimation \cite{hyvarinen_score_matching}.

\end{document}